\documentclass[letterpaper, 9 pt, conference]{ieeeconf}  

\IEEEoverridecommandlockouts                              
\usepackage{graphics} 
\usepackage{epsfig} 
\usepackage{mathptmx} 
\usepackage{times} 
\usepackage{amsmath} 
\usepackage{amssymb}  
\usepackage{mathtools}
\usepackage{float}
\newtheorem{theorem}{Theorem}
\newcommand{\RR}{\mathcal{R}}

\title{\Large \bf
ON RANDOM WEIGHTS FOR TEXTURE GENERATION IN ONE LAYER CNNS
}

\author{Mihir Mongia,$^{1}$ Kundan Kumar,$^{2}$ Akram Erraqabi,$^{3}$ Yoshua Bengio$^{3}$ 
\\$^{1}Stanford ~ University,~^{2}IIT~ Kanpur,~^{3} Montreal~ Institute~ for~ Learning~ Algorithms$
\\$MMONGIA@STANFORD.EDU,~ KUNDAN@IITK.AC.IN,~ ERRAQABI@GMAIL.COM,~ BENGIOY@IRO.UMONTREAL.CA$
}

\begin{document}

\maketitle
\thispagestyle{empty}
\pagestyle{empty}

\begin{abstract}

\texttt{}{Recent work in the literature has shown experimentally that one can use the lower layers of a trained convolutional neural network (CNN) to model natural textures. More interestingly, it has also been experimentally shown that only one layer with random filters can also model textures although with less variability. In this paper we ask the question as to why one layer CNNs with random filters are so effective in generating textures? We theoretically show that one layer convolutional architectures (without a non-linearity) paired with the an energy function used in previous literature, can in fact preserve and modulate frequency coefficients in a manner so that random weights and pretrained weights will generate the same type of images. Based on the results of this analysis we question whether similar properties hold in the case where one uses one convolution layer with a non-linearity. We show that in the case of ReLu non-linearity there are situations where only one input will give the minimum possible energy whereas in the case of no nonlinearity, there are always infinite solutions that will give the minimum possible energy. Thus we can show that in certain situations adding a ReLu non-linearity generates less variable images.}
\\

Key Words - Texture Generation, CNN, Random Weights
\end{abstract}

\section{Introduction}
With the recent advancements in deep learning, many image processing tasks that were once thought impossible are now possible. One such task is that of generating completely random textures that visually look similar to a given sample texture. Recently Gatys et al. [1] experimentally showed that by utilizing a specific energy function that uses the first few layers of a pretrained CNN, natural textures can be effectively generated. Later Champandard [2] used a different energy function and random weights to generate textures. He et al. [3] used the same method as in Gatys et al. [1] with random weights, to generate nice textures. Around the same time, Ustyuzhaninov et al. [4] experimentally showed that less variable natural textures can be generated using just one layer of a CNN with random filters along with the same energy function. Although these methods work quite well there is no theoretical analysis why these methods (in particular the choice of energy functions and also use of random weights) would be helpful for texture generation even in the most basic settings such as in [4].

This leads us to ask a natural question. What theoretical properties lead to the fact that we can generate random textures just using a one layer CNN with random filters? 

In section 3 with a slight adjustment to the one layer CNN architecture, we show rigorously why random weights in one layer CNNs without a nonlinearity, can be used to generate random textures with the same performance as with pre-trained weights,  while also drawing connections to previous work in texture generation. In section 4, we show  how the behavior of generated images changes when one adds a ReLu non-linearity.We show that in the case of a ReLu non-linearity there are often cases where there is only one solution that gives the minimum energy (the input itself), whereas in the case with no nonlinearity there are infinite inputs that give the minimum energy. 
\section{Preliminaries}
We now review a subset of the models used by Ustyuzhaninov et al. [4] that we later analyze theoretically. In particular the subset corresponds to the models which use random weights.

 The model employs a single-layer CNN with standard rectified linear units(ReLus) and convolutions with stride one, no bias and padding (f-1)/2 where f is the filter-size (f is always an odd number). This choice of padding will ensure that the spatial dimension of the output is the same as the spatial dimension of the input. In addition, 363 filters of dimensions 11 $\times$ 11 $\times$ 3 (filter width, filter height, number of input channels) are randomly generated from a uniform distribution according to [5]. 

 For an original texture image $x$, a matrix $G^{x}$ 
  is formed,

 \begin{equation}
 G_{ij}^{x} = \sum_{m=1}^{m=M}F_{im}F_{jm}
 \end{equation}
where $F_{ij}$ corresponds to the output of the ith filter (after the nonlinearity) at location j. 

To generate new textures corresponding to an original image $x$, an energy function $E$ is developed, 

 \begin{equation}
 E(y)= ||G^{x} - G^{y}\,||_F
 \end{equation}
An image y is initialized randomly to values between 0 and 1. Then $y$ is changed according to gradient descent until $y$ is a local minimum with respect to the energy function. Note that the smallest this energy can be is zero. 

\section{Theory for using Random Weights}

In this section we consider the same model as above with two modifications. We consider a model where there is no non-linearity. We also use circular convolution rather than valid convolution as done in Saxe et al. [6]. We study the behavior of the algorithm without a ReLu non-linearity to gain insight into what is happening or not happening in the non-linear case. In addition, if the random weights in the model of Ustyuzhaninov et al. [4] were to be positive, then having a ReLu nonlinearity would be equivalent to using no non-linearity. This is because an image signal comes in values between 0 and 255. Thus the convolution of positive filters and an image signal will be positive. We use circular convolution, as in [6], rather than valid convolution because we can more easily get a theoretical analysis. Intuitively doing circular convolution in a texture image is not too different than doing valid convolution because the statistics are usually uniform across texture images. 

An important fact to note is that the output of a circular convolution of an original image and filter $F_{i}$ can be written as a matrix multiplication where the matrix is a function of the filter $F_{i}$. These matrices are called block circulant with circulant blocks (BCCB) matrices. These matrices can be diagonalized in the form

\begin{equation}
F_{i} =  UD_{i}U^\dag 
\end{equation}
where $U$ is the discrete fourier basis in two dimensions [7].Thus the output of convolving an image $x$ with a filter $F_{i}$ is 

\begin{equation}
F_{i}(x) =UD_{i}U^\dag x
\end{equation}
where $x$ has been appropriately vectorized, and an entry of $G^{x}$ can be re-expressed in the following way,

\begin{equation}
\begin{split}
&G_{ij}^{x}= x^{T} UD_{i}^\dag U^\dag UD_{j}U^\dag x
=x^{T} UD_{i}^\dag D_{j}U^\dag x
=x^{T} UD_{ij}U^\dag x\\
&~~~~~~~~~~~~~~~~~~~~~~~~~~~~= \sum_{k}|\lambda_{k}^{x}|^2 D_{ij}^{k}\\
\end{split}
\end{equation}
where $\lambda_{k}^{x}$ is the is the kth fourier coeffecient corresponding to the kth basis vector in $U$ such that

\begin{equation}
x= \sum_{k}\lambda_{k}^{x}  U(:,k)
\end{equation}
	
    Ideally we would like to generate all the images $y$ such that $G^{x} = G^{y}$. This would correspond to an energy of zero. From the equations above we can see that finding a $y$ such that $Energy(y)=0$ would come down to solving a system of linear equations. To be clear the system of equations would be 

\begin{equation}
\begin{split}
&G_{11}= \sum_{k}|\lambda_{k}^{y}|^2 D_{11}^{k},
~G_{12}= \sum_{k}|\lambda_{k}^{y}|^2 D_{12}^{k},
~G_{13}= \sum_{k}|\lambda_{k}^{y}|^2 D_{13}^{k}\\
& ~~~~~~~~~~~~~ \dots 
~~~G_{nn}= \sum_{k}|\lambda_{k}^{y}|^2 D_{nn}^{k}\\
\end{split}
\end{equation}
In matrix form the linear system of equations would look like

\begin{equation}
g= M |\lambda^{y}|^2
\end{equation}
or 
\begin{equation}
\begin{pmatrix}
G_{1,1}^{x}\\
G_{1,2}^{x} \\
\vdots   \\
G_{2,1}^{x} \\
\vdots  \\
G_{n,n}^{x} \\
\end{pmatrix}=\begin{pmatrix}
D_{1,1}^{1} & D_{1,1}^{2} & D_{1,1}^{3} & \dots &  D_{1,1}^{n}\\
D_{1,2}^{1} & D_{1,2}^{2} & D_{1,2}^{3} & \dots &  D_{1,2}^{n}\\
\vdots  & \vdots  & \vdots  & \dots  & \vdots \\
D_{2,1}^{1} & D_{2,1}^{2} & D_{2,1}^{3} & \dots &  D_{2,1}^{n}\\
\vdots  & \vdots  & \vdots  & \dots  & \vdots \\
D_{n,n}^{1} & D_{n,n}^{2} & D_{n,n}^{3} & \dots &  D_{n,n}^{n}\\

\end{pmatrix}
|\lambda^{y}|^2
\end{equation}

First we note that this system of equations is solving for the magnitude of the fourier coefficients of a particular signal. Thus this system puts no constraints on the phase of each fourier coefficient and thus, even in the case of M being full rank, there are always infinite solutions to equation 9. 

In the case that the system of equations is not full rank, then the choice of filters clearly effects what type of solutions will give energy zero. This is because the vector $|\lambda^{y}|^2$ is unique up to the rows of M. This has some nice implications. Let's define 

\begin{equation}
d_{ij} = diag(D_{i}^\dag D_{j})
\end{equation}
Since we have constraints of the form

\begin{equation}
G_{ii} = d_{ij}^{T} |\lambda^{y}|^2
\end{equation}
it implies the possible solutions of $|\lambda^{y}|^2$ must lie on the $n-1$ dimensional space where equation 11 holds.Thus $|\lambda^{y}|^2$ lies within an intersection of roughly $n^2/2$, $n-1$ dimensional spaces that have the same projection on to the frequency domain of the filters as does $|\lambda^{x}|$. Notice if this low rank system were due to random weights then the possible solutions would lie within a different intersection of $n-1$ dimensional spaces. Thus using random weights would yield distinctly different images than would using pre-trained weights in the case where the number of filters is not large enough to create a full rank matrix $M$.  

Notice that even if this system of equations is full rank, the solution to the linear system only gives us the magnitude of each fourier coefficient. Also notice that the number of equations is approximately $n^2/2$. Thus if we have a relatively small number random filters compared to the dimension of $x$, the matrix M can become full rank. If M is full rank and M is created either by random weights or pretrained weights, the solutions $y$ that yield energy zero are characterized by images for which the fourier coefficient magnitudes are the same as the fourier coefficient magnitudes of the image $x$. This is because we already know the fourier coefficients of the image $x$ already satisfy the equations from 7 and thus $|\lambda_{k}^{y}| = |\lambda_{k}^{y}|$ for any $k$. Thus in the case that there are enough random weights and pre-trained weights , the solution $y$ to the system of equations is no different than the solutions $y$ to the system of equations generated by pretrained weights.

In Figure 1 we show textures generated by simply randomizing phase of fourier coefficients. one can notice that the ridges happen in completely different places. One issue with this however is that one can see vertical and horizontal lines appearing in the generated textures. This is due to the fact that the FFT assumes images are periodic and in fact may suggest the need for nonlinearities.
\begin{figure}[h!]
\centering
\includegraphics[width=7cm]{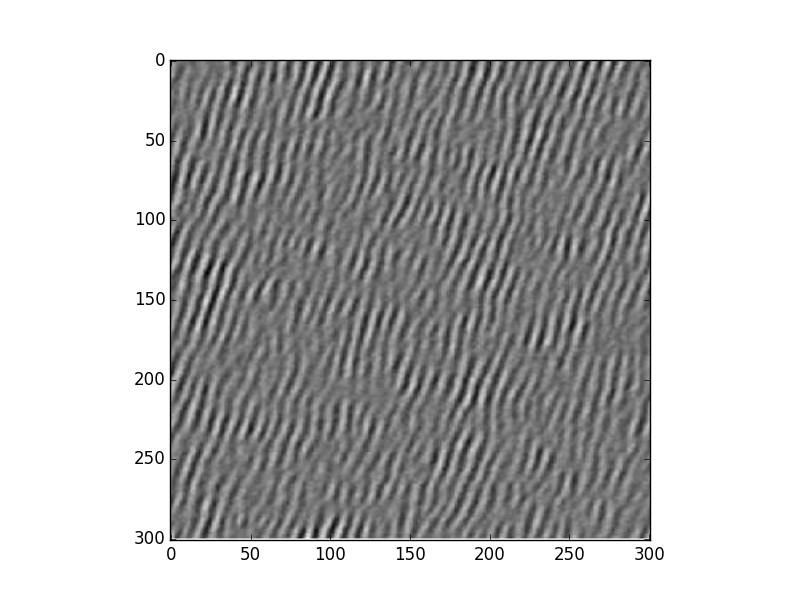}
\includegraphics[width=7cm]{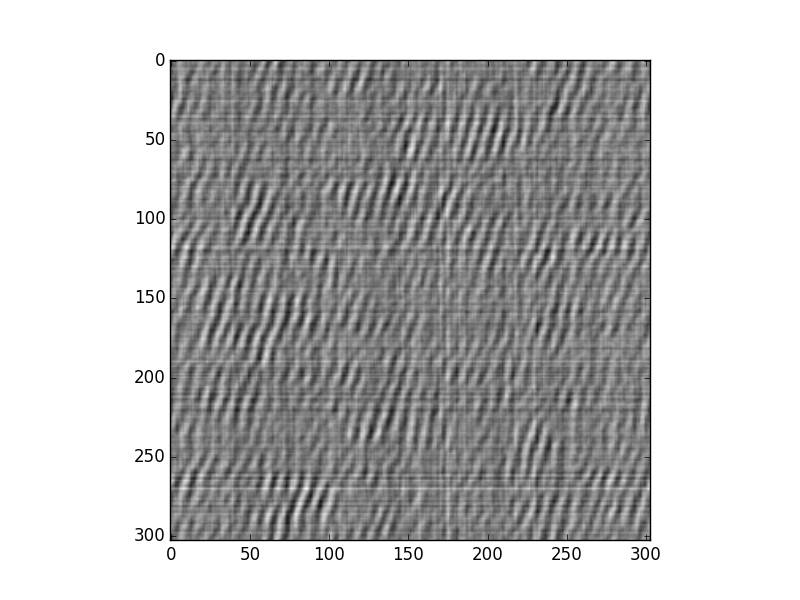}
\includegraphics[width=7cm]{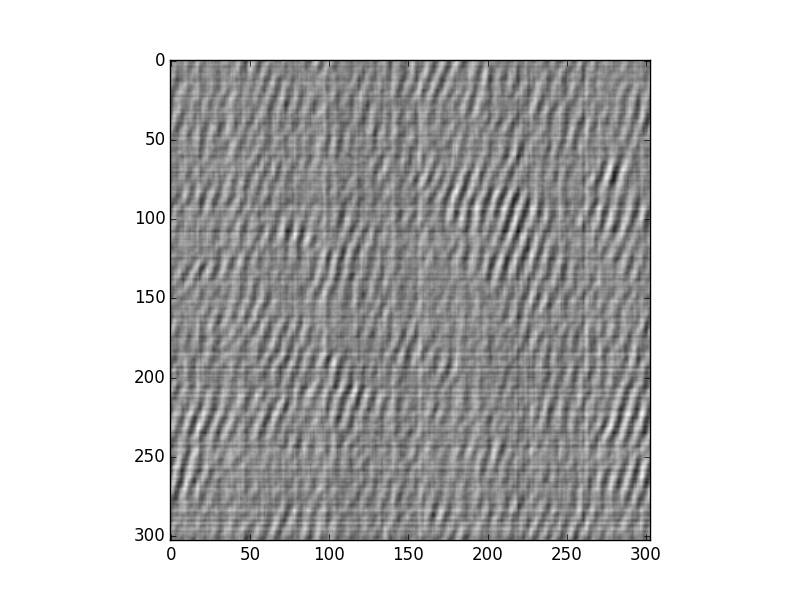}

\caption{This figure shows how randomly assigning phase to fourier coefficients can generate random textures. The top image is the original black and white texture. The bottom two images are generated textures.}
\end{figure}
\section{Unique Solutions in the ReLu case with Circular Convolution}
In the case where there is no nonlinearity, we show that in all cases there are infinite solutions that will yield energy equal to zero (The results above also hold for 1-D vector signals by simply letting $U$ be discrete fourier transform for one dimensional signals). Since in [4] a ReLu is used, the results above beg the question, does the same behavior hold in the case where there is a ReLu non-linearity. We show below that for almost all 1-D vector signals(as opposed to 2-D image signals) if one samples enough random filters that there is only one solution that yields energy zero.


First we set some notation. For an $x \in R^{n}$, let $C(x)$ denote the circulant matrix corresponding to x. In other words the first row of $C(x)$ is x. The second row is x circularly shifted by one and so on. Let the set $P_{i}(x)$ be all the vectors  $v$ such that in the matrix-vector multiplication $C(x)v$ has only the $ith$ element greater than zero. 

\begin{theorem}
Let $c_{1}$, $c_{2}$, ..., $c_{n} \in R^{n}$ be linearly independent vectors. Then the intersection of the half planes decided by these n vectors has non-zero volume. 
\end{theorem}

\begin{proof}
Denote the half plane restrictions by 
\begin{equation}
Cx < 0 
\end{equation}
where the rows of C correspond to the vectors setting the half planes. $C$ is full rank because the rows are linearly independent. Now let  
\begin{equation}
q_{k} = C^{-1}e_{k}
\end{equation}
where $e_{k}$ are the canonical basis vectors of a euclidean vector space. Notice that as a result the $q_{k}$ will be linearly independent and thus form a basis for $R^{n}$ because there are $n$ such vectors. Thus any $x$ of the form 

\begin{equation}
x = \alpha_{1}q_{1} + \alpha_{2}q_{2} + ... \alpha_{n}q_{n}
\end{equation}
where $\alpha_{i} > 0$, satisfies the half plane restrictions. Since the region of possible $x$ is a convex cone constructed from basis vectors that span the whole euclidean space, the region of vectors x that satisfy the half-plane restrictions has non-zero volume.  
\end{proof}
Corollary: As long as $C(x)$ is full rank, $P_{i}(x)$ is a set that has nonzero volume in $R^{n}$ for any $i$. In other words, if you were to uniformly sample vectors from the unit ball, there is a finite probability that these vectors would be in the set $P_{i}(x)$ for any $i$. 

\begin{theorem}
For any $G^{x}$ there are n vectors $y$ such that $G^{x} = G^{y}$. In particular, any y that is just a circularly shifted version of x will have a gram matrix $G^{y} = G^{x}$. This holds for any kind of non-linearity used in the network.
\end{theorem}

\begin{proof}
Consider $F^{i}$ to be the circulant matrix corresponding to i-th convolution filter. Then, output feature map corresponding to i-th convolution filter for input vector X can be written as $F^{i}X$. We'll use h(.) to represent the non-linearity that can be applied to any scalar or element-wise to any vector/matrix. Using this notation, we can say that 
\begin{equation}
\label{gram_eq}
\begin{split}
G^{X}_{ij} &= (h(F^{i}X))^\top(h(F^{j}X))
	= \sum_{k=1}^{n} h(<(F^i_{k,:},X>). h(<F^j_{k,:},X>)
\end{split}
\end{equation}.
where $F^i_{k,:}$ is the $k^{th}$ row of $F^i$.

Now consider an operator Cs such that, Cs(x) circularly shifts x by one e.g. if $x = [x_1, x_2, x_3, ..., x_n]$, then $Cs(x) = [x_n, x_1, x_2, x_3, ..., x_{n-1}]$. Here, n is the length of vector x, $<a,b>$ is the dot product between vectors a and b. Then, $F^{i}$ being circulant implies that $Cs(F^i_{k,:}) = F^i_{k+1,:} \forall k \in [1, n-1]$ and $Cs(F^i_{n,:}) = F^i_{1,:}$. Also, it is straightforward to see $<Cs(a),Cs(b)> = <a,b> \forall a,b \in R^n$. Hence, we have
\begin{equation}
\label{f1}
\begin{split}
<F^i_{k+1,:}, Cs(X)> & = <Cs(F^i_{k,:}), Cs(X)>= <F^i_{k,:}, X> \\
\implies h(<F^i_{k+1,:}, Cs(X)>) & = h(<Cs(F^i_{k,:}), Cs(X)>) ~~\forall k \in [1,n-1]
\end{split}
\end{equation}
and
\begin{equation}
\label{f2}
\begin{split}
<F^i_{1,:}, Cs(X)> & = <F^i_{n,:}, X>\\
\implies h(<F^i_{k+1,:}, Cs(X)>) & = h(<Cs(F^i_{k,:}), Cs(X)>)
\end{split}
\end{equation}
This is true for all i in range [1,N] where N is the number of filters.
Hence, from eq. \ref{gram_eq} and eq. \ref{f1}, \ref{f2}, we have 
\begin{equation}
\begin{split}
G^{Cs(X)}_{ij} & = h(<F^i_{1,:},Cs(X)>)* h(<F^j_{1,:},Cs(X)>) \\
    &  + \sum_{k=2}^{n} h(<F^i_{k,:},Cs(X)>)* h(<F^j_{k,:},Cs(X)>)\\
	& = h(<F^i_{n,:},X>)* h(<F^j_{n,:},X>) \\
    & + \sum_{k=1}^{n-1} h(<F^i_{k,:},X>)* h(<F^j_{k,:},X>)
  = G^{X}_{ij}
\end{split}
\end{equation}
Hence, $G^{X} = G^{Cs(X)} = G^{Cs(Cs(X))} .. = G^(Cs^m(X))$. Hence, we get at least as many solutions as the possible circulations of x.
\end{proof}
\begin{theorem} Using the same notation set before Theorem 1, suppose we have $n$ sets of filters. Let each set $S_{j}$ contain $n$ filters (or equivalently vectors) each of which $\in P_{i}(x)$ for a unique $i$. In others words, if $v,z \in S_{j}$, then if $v \in P_{i}(x)$ for some $i$ then $z \not\in P_{i}(x)$.  Denote the set $S_{P_{i}}$ as the set of vectors  $ \in (\cup_{j=1}^{n}S_j) \cap P_{i}(x)$. Let the vectors in each $S_{P_{i}}$ be linearly independent. Under these conditions, if $G^{x} = G^{y}$ for some vector $y$, then $y$ is equal to $x$ or some circularly shifted version of $x$.
\end{theorem}

\begin{proof}
We denote the ReLU activation as $\RR(x)$.
Let's consider the set of $n^2$ filters in $\cup_{j=1}^{n}S_j$. For a given $x$, we want to characterize the solutions $y$ of the equation $G^{y} = G^{x}$. This equation is actually equivalent to a system of equations that correspond to the Gram matrix components 
$$ \forall i,j \in [1,\cdots,K] ~~~~~~ G^{y}_{ij} = G^{x}_{ij}$$
We can rewrite $ G^{x}_{ij} = \RR(F^i x)^{\top} \RR(F^j x)$.
Let's consider a diagonal $n\times n$ block of $G^x$ which corresponds to the filters coming from $S_{i}$ for some $i$. Without loss of generality suppose we have ordered the entries of $G$ such that the first diagonal block corresponds to filters in $S1$, the second diagonal block corresponds to the filters in $S2$ and so on.  Without loss of generality let's consider the first block. We know that the corresponding filters in $S_{1}$ are defined such that only one component of $F^{q}_{S_{1}} x$ is positive, i.e $\forall q ~~\|\RR(F^{q}_{S_{1}} x)\|_0 = 1$.Here $q$ is indexing the $n$ filters in $S_{1}$. Due to the conditions assumed in the theorem ($S_{j}$ contains $n$ filters each of which $\in P_{i}(x)$ for a unique $i$), we also have that $\RR(F^{q1}_{S_{1}} x)^{\top} \RR(F^{q2}_{S_{1}} x) = G^{S1}_{q1q2}*\delta_{q1q2}$ where $G^{S1}_{q1q2}$ corresponds to the appropriate entry in the first diagonal block of $G^{x}$. Note that this implies the first diagonal block of $G^{x}$ is diagonal because of the $\delta_{q1q2}$ which is only non-zero when $q1 = q2$.

Given that we want $G^x = G^y$, it implies the first diagonal $n\times n$ block of $G^y$  must also be diagonal. This implies that for any $F_{S1}^{q} \in  S_{1}$, $ \|\RR(F^{q}_{S_{1}} y)\|_0 = 1$ because if it were greater than 1, then the $n\times n$ block of $G^y$ will have non diagonal non-zero elements. If $ \|\RR(F^{q}_{S_{1}} y)\|_0 = 0$, then of course the $n\times n$ block would be the zero matrix.

This means that the only equations that characterize the solutions are those that correspond to the diagonal terms of the first block of $G^x$. These equations are $ \|\RR(F^{q}_{S1} y)\|_2^2 = G^{S1}_{qq}$. Since $\|\RR(F^{q}_{S1} y)\|_0 = 1$, there are $n$ possible equations $F^{q}_{S1}[k,:]y = \sqrt{G^{S1}_{qq}}$ that could generate consistent vectors $y$, where $k$ corresponds to some row in the matrix (filter) $F_{S1}^{q}$.

Performing the same reasoning using the other sets of filters $S_{j}$, we end up with $n$ possible equations for each filter that characterize a solution y. $F^{q}_{S_{j}}[k,:] y = \sqrt{G_{qq}^{S_{j}}}$.

Note that the non-diagonal blocks of $G^x$ happen to be the inner product of the output of filters coming from $S_{i} ~ and ~ S_{j}$ for some $j$ and $i$. These blocks contain also one non zero value per row for the same reason that make the diagonal blocks diagonal matrices. The value of knowing for which pairs of $F^{q1}_{S_{i}}$ and $F^{q2}_{S_{j}}$ produce a non-zero value is that it lets us know which matrices are aligned. In other words if the entry in $G^{x}$ corresponding to $F^{q1}_{S_{i}}$ and $S^{q2}_{S_{j}}$ is greater than zero, then if $F^{q2}_{S_{j}}[k,:] y = \sqrt{G^{S_{j}}_{q2q2}}$, then $F^{q1}_{S_{i}}[k,:] y= \sqrt{G^{S_{i}}_{q1q1}}$. Notice we used the same index $k$. Suppose we have aligned the above 2 filters. Then we can align those with $n-2$ other filters coming from other $S_{j}$, by looking in the appropriate non diagonal blocks of $G^{x}$.

Since we have aligned the filters appropriately, we now have n possible system of equations since we can choose n values of k. We note under the conditions specified in Theorem 3, that these system of equations is full rank, and can thus only have one solution. One can also see that each of these system of equations just gives a circularly shifted solution of $x$. This phenomenon corresponds to theorem 2.
\end{proof}
\begin{theorem}
Suppose we generate $nk$ random filters from the infinity norm 1 ball and suppose $C(X)$ is full rank. As $k$ approaches infinity, the probability that there is unique solution to $G^{x} = G^{y}$ (up to a circular shift) approaches one.
\end{theorem}
\begin{proof}
Let $\delta_1...\delta_n$ be the probability that a random vector lands in $P_{1}(x) ... P_{n}(x)$ respectively. We know $\delta_1...\delta_n$ is greater than zero because of theorem 1. Let $k = c*n$ where $c$ is some rational number. Let $count_{i}$ be the number of times we get filters in $P_{i}(x)$ in the $ith$ set of $k$ draws.
\begin{equation}
\begin{split}
&P(count_{i}<n) = P(count_{i} < (\frac{1}{c\delta_{i}})cn\delta_{i}) \\
&= P(count_{i} < (1 -\frac{c\delta_{i}-1}{c\delta_{i}})cn\delta_{i})\\ 
&< exp(-\frac{c\delta_{i}-1}{2c\delta_{i}} \frac{c\delta_{i}-1}{c\delta_{i}} c  \delta_{i}n)
\end{split}
\end{equation}
This follows from classical chernoff bound results.
Notice the term $\frac{c\delta_{i}-1}{c\delta_{i}}$ approaches 1 as $c$ approaches infinity.

Thus 

\begin{equation}
\begin{split}
&P(count_{1}>n, count_{2} >n, \dots count_{n} > n) \\
&= \prod_{i=1}^{n} (1 - exp(-\frac{c\delta_{i}-1}{2c\delta_{i}} \frac{c\delta_{i}-1}{c\delta_{i}} c  \delta_{i}n))
\end{split}
\end{equation}
which approaches 1 as c approaches infinity. In addition, given that we have sampled more than $n$ vectors that are in $P_{i}(x)$ for some $i$, if we choose n of those vectors they are going to be linearly independent with probability one since the set of linearly dependent vectors in $P_{i}(x)$ has measure zero.Thus as $c$ approaches infinity we have the conditions noted in theorem 3 with probability one. One can find the sets $S_{1}, S_{2}, \dots S_{n}$ through exhaustive search of the Gram Matrix and then execute the proof for theorem 3 and come to the conclusion that if $G^{x} = G^{y}$ for some vector $y$, then $y$ is equal to $x$ or some circularly shifted version of $x$. 
\end{proof}

\section{Connections to Previous Work}
Galerne et al. [8] show that an algorithm called Random Phase Noise (RPN) is effective at generating realistic textures of a certain class. In the most basic version of their algorithm, the authors take texture like images and generate new textures simply by randomizing the phase of the fourier coefficients.As long as there are enough filters to create a full rank system, the slightly modified model in section 3 is mathematically equivalent to randomizing phase. Thus we expect that even our modified model(with no nonlinearity) should produce good random textures, which is in fact what we observe in Figure 1.  The authors note the same horizontal and vertical line artifacts with this algorithm as we do in our experiments and thus do some extra image processing to get textures that do not have line artifacts.

 \section{Discussion}
We have shown rigorously that in a slightly simplified model compared to that of [4]   (with circular convolution and no ReLu nonlinearity), there is no need for pretrained weights. To be precise, if the number of weights is small, then the images generated with pretrained weights and random weights will be different from each other. However, as the number of weights increases, random or pretrained weights make no difference in the images generated. We show that the images generated from a large number of weights are just modified versions of the original image. Frequency coefficient magnitudes are preserved, whereas frequency coefficient phases are randomized.

Using the lack of a nonlinearity as a starting point and the fact that in [4] a Relu nonlinearity is used, we seek to see if having a ReLu would generate a set of signals with similar properties as in the no ReLu case. We show rigorously in the one-dimensional case that there are conditions under which the only signals $x^{,}$ that can be generated with the same gram matrix as the original signal $x^{*}$ are circularly shifted versions of $x^{*}$. This is distinctly different from the linear case where there would be an infinite number of solutions.

We view these results as giving a theoretical starting point to understand the contribution of nonlinearities, random weights, as well as the gram matrix energy function in texture generation. To be clear though, there is room for future work. Since the energy function is non-convex with respect to input images, input images may be caught in a local minima. We have no way to get a handle on what type of images these local minima might produce. Secondly we have not shown in full generality what happens in the Relu case. We have shown just sufficient conditions under which the solutions to a Relu model and linear model differ greatly. Thirdly our analysis does not yet account for multiple layers. On the other hand, the analysis presented above may provide guiding intuition for multiple layers since the output of each layer is just a linear convolution and ReLu of the previous layer. 

\end{document}